\documentclass[twoside,11pt]{article}
\usepackage{pgm}

\usepackage[utf8]{inputenc}
\inputencoding{utf8}

\usepackage{algorithm}
\usepackage[noend]{algpseudocode}
\algrenewcommand\Return{\State \algorithmicreturn{} }
\algnewcommand{\LineComment}[1]{\State \(\triangleright\) #1}
\usepackage[fleqn]{amsmath}
\usepackage{subcaption}

\usepackage{hyperref,color,soul,booktabs}
\setulcolor{blue}
\newcommand{\BibTeX}{\textsc{B\kern-0.1emi\kern-0.017emb}\kern-0.15em\TeX}

\begin{document}

\title{Sum-Product Network Decompilation}

\author{\Name{Cory J. Butz} \Email{butz@cs.uregina.ca}\\
   \addr University of Regina, Regina, Canada\and
   \Name{Jhonatan S. Oliveira} \Email{oliveira@cs.uregina.ca}\\
   \addr University of Regina, Regina, Canada\and
   \Name{Robert Peharz}  \Email{rp587@cam.ac.uk}\\
   \addr Eindhoven University of Technology, Eindhoven, Netherlands\\
   \addr University of Cambridge, Cambridge, United Kingdom}

\maketitle

\begin{abstract}
    There exists a dichotomy between classical probabilistic graphical models, such as \emph{Bayesian networks} (BNs), and modern tractable models, such as \emph{sum-product networks} (SPNs).
    The former generally have intractable inference, but provide a high level of interpretability, while the latter admit a wide range of tractable inference routines, but are typically harder to interpret.
    Due to this dichotomy, tools to convert between BNs and SPNs are desirable.
    While one direction -- compiling BNs into SPNs -- is well discussed in Darwiche's seminal work on \emph{arithmetic circuit} compilation, the converse direction -- \emph{decompiling} SPNs into BNs -- has received surprisingly little attention.
    In this paper, we fill this gap by proposing SPN2BN, an algorithm that decompiles an SPN into a BN.
    SPN2BN has several salient features when compared to the only other two works decompiling SPNs.
    Most significantly, the BNs returned by SPN2BN are \emph{minimal} independence-maps that are more parsimonious with respect to the introduction of \emph{latent} variables.
    Secondly, the output BN produced by SPN2BN can be precisely characterized with respect to a compiled BN.
    More specifically, a certain set of directed edges will be added to the input BN, giving what we will call the \emph{moral-closure}.
    Lastly, it is established that our compilation-decompilation process is idempotent.
    This has practical significance as it limits the size of the decompiled SPN.
\end{abstract}

\section{Introduction}
\label{sec:introduction}

There exists a trade-off between classical probabilistic graphical models and recent tractable probabilistic models.
Classical models, such as \emph{Bayesian networks} (BNs) \cite{pear88}, provide high-level interpretability as conditional independence assumptions are directly reflected in the underlying graphical structure.
However, the downside is that performing exact inference in BNs is NP-hard \cite{coop90}.
In contrast, modern tractable probabilistic models, such as \emph{sum-product networks} \cite{darwiche09,poon2011sum}, allow a wide range of tractable inference, but are harder to interpret.
In order to combine advantages of both BNs and SPNs -- which are complementary regarding interpretability and inference efficiency -- tools to convert back and forth between these types of models are vital.

The direction compiling BNs into SPNs is well understood due to Darwiche's work on BN compilation into \emph{arithmetic circuits} (ACs) \cite{darwiche09}.\footnote{ACs and SPNs are equivalent models. \emph{Deterministic} models are typically referred to as ACs, while \emph{non-deterministic} models are called SPNs. See Section \ref{sec:back} for further details.}
Since inference in ACs and SPNs can be performed in linear time of the network size, compilation amounts to finding an inference machine with minimal inference cost.
ACs can also take advantage of \emph{context-specific-independence} \cite{bout96} in the BN parameters to further reduce the size of the AC.

The converse direction of SPN decompilation into BNs has received limited attention.
This lack of attention can be understood historically.
Since an original purpose of ACs was to serve as efficient inference machine for a \emph{known} BN, decompilation would seem like a mere academic exercise. 
The proposition of SPNs, however, introduced some practical changes to the AC model.
First, unlike ACs, SPNs are typically learned directly from data, i.e., a reference BN is not available.
Thus, providing a corresponding BN would greatly improve the interpretability of the learned SPN.
Second, as already mentioned, SPNs are typically non-deterministic, which naturally introduces an interpretation of SPNs as hierarchical latent variable models \cite{Peharz:2016wl,choi2017relaxing}.
A decompilation algorithm for SPNs should account for this fact, and generate BNs with a plausible set of latent variables.
Thus, naively decompiling SPNs with a decompilation algorithm devised for ACs would yield densely connected and rather uninterpretable BNs.

In this paper, we fill this void by formalizing SPN decompilation.
We propose SPN2BN, an algorithm that converts a trained SPN into a BN.
Our algorithm arguably improves over the only two other approaches in the literature addressing the connection between BNs and SPNs \cite{zhao2015relationship} and \cite{Peharz:2016wl}.
First, while \cite{Peharz:2016wl} produce a BN for a given SPN, these BNs are, in general, not minimal \emph{independence-maps} (I-maps) \cite{pear88}, i.e., they introduce needless dependence assumptions. 
Our algorithm SPN2BN, on the other hand, produces \emph{minimal} I-maps.
Second, \cite{zhao2015relationship} is excessive with the number of introduced latent variables.
In fact, both approaches interpret each single sum node in an SPN as a latent variable on its own.
In this paper, we devise a more economical approach and identify groups of sum nodes to jointly represent one latent variable.
This grouping is based on whether sum nodes are ``on the same level of circuit hierarchy'' and ``responsible'' for the same set of observable variables (these notions will be made formal in Section \ref{sec:dec}).

These design choices for SPN2BN improve over \cite{zhao2015relationship} and \cite{Peharz:2016wl} both in terms of a reduced number of BN nodes (latent variables) and a reduced number of edges (minimal I-mapness).
While this design leads to more succinct and perhaps more esthetic BNs, SPN2BN is also justified in a formal way.
We show that SPN2BN can be seen as the inverse of the SPN compilation process proposed in \cite{darwiche2003differential}.
Consider an arbritary BN ${\cal B}$ that was compiled into an AC with \emph{variable elimination} following a \emph{reverse topological order} (VErto).
Convert the AC into an SPN ${\cal S}$ with an optional marginalization operation\footnote{SPNs are closed under marginalization, that is, any sub-marginal of any SPN can again be represented as an SPN.} \cite{darwiche09} creating latent variables by rendering random variables unobserved.
Then, decompiling SPN ${\cal S}$ with SPN2BN yields a BN ${\cal B}^c$ with a set of directed edges that are a superset of the original BN, which we call the \emph{moral closure} of ${\cal B}$.
The SPN2BN algorithm is consistent with respect to this compilation algorithm, in the sense that it always yields the moral closure ${\cal B}^c$ of any given BN ${\cal B}$.
Repeating the compilation-decompilation process with the morally closed BN is then idempotent, i.e.~compilation and decompilation are consistent inverses of each other.
Consistency with a compilation procedure is arguably a desirable property as it serves as a characterization of the
decompilation method.
In contrast, \cite{zhao2015relationship} and \cite{Peharz:2016wl} are not consistent with any general-purpose compilation algorithm, and tend to excessively increase the number of variables and edges in the constructed BN.
Lastly, even when the input SPN does not stem from an assumed compiler, e.g., when it is learned from data, the VErto compilation assumption within SPN2BN helps us to interpret the result of decompilation.

For example, consider a prominent example of a BN in Figure \ref{fig:comp_hmm} (\subref{subfig:bn_hmm}), commonly known as a hidden Markov model (HMM).
In Figure \ref{fig:comp_hmm} (\subref{subfig:spn_hmm}), we see the result of VErto, followed by marginalization of $H_1$, $H_2$, and $H_3$ deeming these three variables latent.
This SPN shall be converted back into a BN.
In Figure \ref{fig:decomp_hmm} (\subref{subfig:zhao_hmm}, \subref{subfig:peharz_hmm}) we see the BNs produced by \cite{zhao2015relationship} and \cite{Peharz:2016wl}, respectively.
Both BNs introduce more variables than were present originally, and the introduced edges hardly reflect the succinct independence assumptions of the HMM.
In Figure \ref{fig:decomp_hmm} (\subref{subfig:us_hmm}), the decompilation result by our SPN2BN algorithm is depicted.
It can be seen that SPN2BN recovers the original HMM structure, where latent variables $Z_1$, $Z_2$, and $Z_3$ exactly correspond to the original latent variables $H_1$, $H_2$ and $H_3$, respectively.
Evidently, no decompilation is able to recover the original labels for these variables, since reference to these has been explicitly removed by the previous (optional) marginalization operation.
However, we see that SPN2BN successfully detects their signature in the compiled SPN, enabling it to recover an equivalent set of latent variables. 


\begin{figure}[tb]
    \begin{center}
    \begin{subfigure}[t]{0.20\linewidth}
      \centering
        \includegraphics[width=\linewidth]{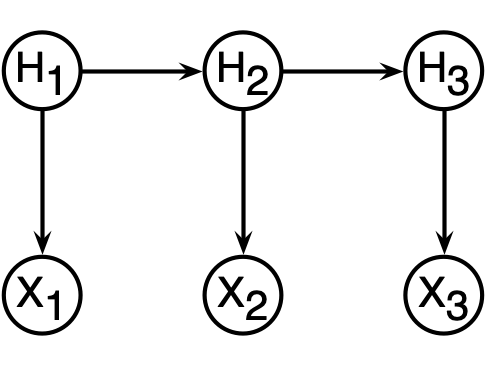}
           \caption{}
           \label{subfig:bn_hmm}
    \end{subfigure}
    \begin{subfigure}[t]{0.30\linewidth}
      \centering
        \includegraphics[width=\linewidth]{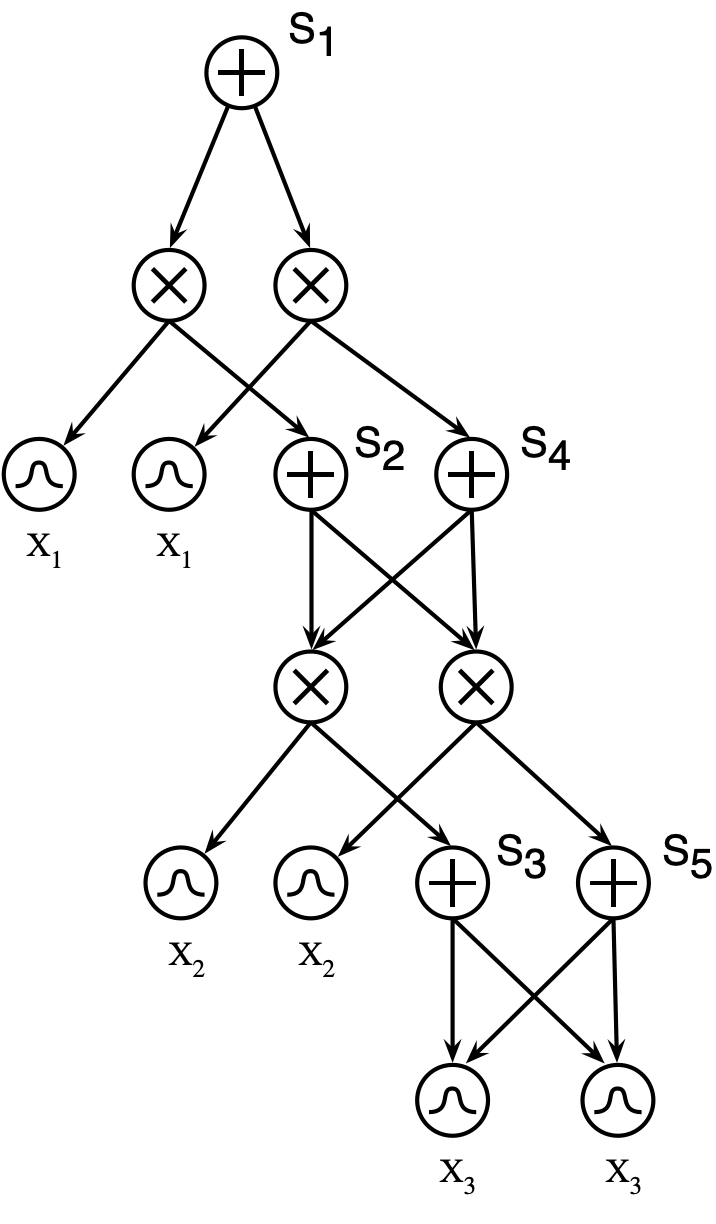}
           \caption{}
           \label{subfig:spn_hmm}
    \end{subfigure}
    \caption{Compilation of the BN in (\subref{subfig:bn_hmm}) using VErto \protect\cite{darwiche09}, and marginalizing $H_1$, $H_2$, and $H_3$, yields the SPN in (\subref{subfig:spn_hmm}).}
    \label{fig:comp_hmm}
    \end{center}
  \end{figure}

  \begin{figure*}[tb]
    \begin{center}
    \begin{subfigure}[t]{0.25\linewidth}
      \centering
        \includegraphics[width=\linewidth]{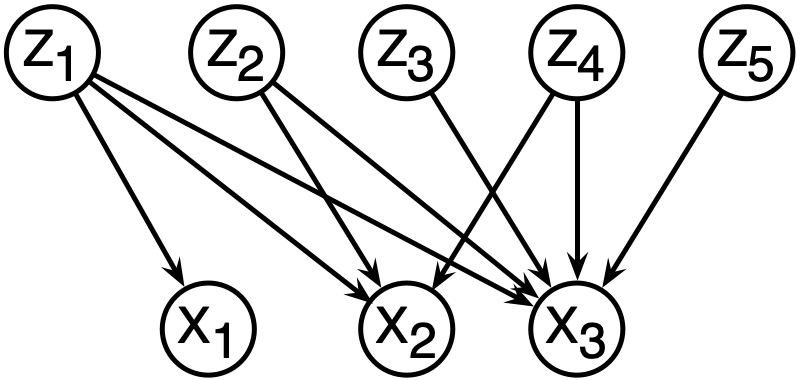}
           \caption{}
           \label{subfig:zhao_hmm}
    \end{subfigure}
    \begin{subfigure}[t]{0.22\linewidth}
      \centering
        \includegraphics[width=\linewidth]{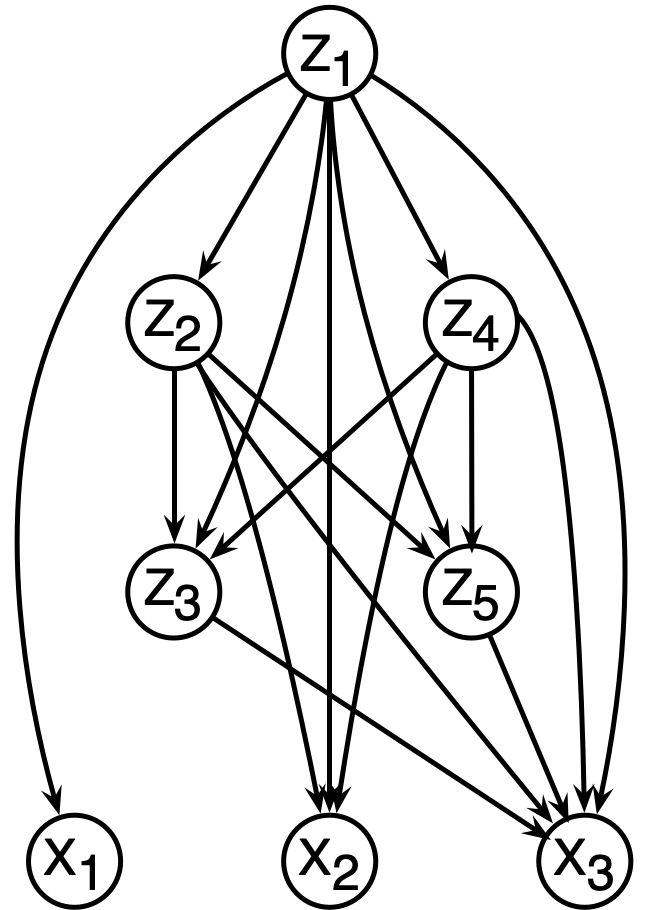}
           \caption{}
           \label{subfig:peharz_hmm}
    \end{subfigure}
    \begin{subfigure}[t]{0.20\linewidth}
      \centering
        \includegraphics[width=\linewidth]{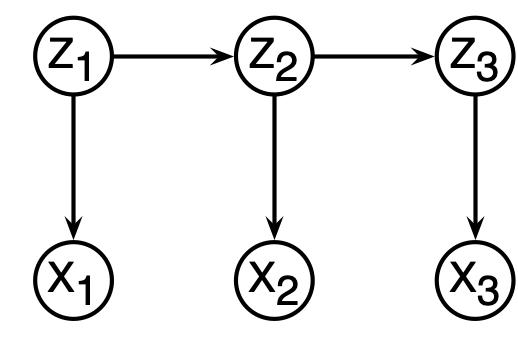}
           \caption{}
           \label{subfig:us_hmm}
    \end{subfigure}
    \caption{Decompilation of the SPN in Figure \ref{fig:comp_hmm} (\subref{subfig:spn_hmm}) by \protect\cite{zhao2015relationship} in (\subref{subfig:zhao_hmm}), \protect\cite{Peharz:2016wl} in (\subref{subfig:peharz_hmm}), and SPN2BN in (\subref{subfig:us_hmm})}
    \label{fig:decomp_hmm}
    \end{center}
  \end{figure*}
\section{Sum-Product Networks}
\label{sec:back}

Here, we review BNs, ACs, and SPNs, as well as the compilation of BNs into SPNs.

We denote \emph{random variables} (RVs) by uppercase letters, such as $X$ and $Y$, possibly with subscripts, and their values by corresponding lowercase letters $x$ and $y$.
Sets of RVs are denoted by boldfaced uppercase letters and their combined values by corresponding boldfaced lowercase letters.
The \emph{children} of a variable $V$ in a \emph{directed acyclic graph} (DAG) ${\cal B}$, denoted $Ch(V)$, are the immediate descendants of $V$ in ${\cal B}$.
Similarly, the \emph{parents} $Pa(V)$ of a variable $V$ are immediate ancestors of $V$ in ${\cal B}$.
The \emph{descendants} $De(V)$ are the variables $V^\prime$ with a directed path from $V$ to $V^\prime$ in ${\cal B}$.
The \emph{ancestors} $An(V)$ of $V$ are similarly defined.
A variable $V_k$ is called a \emph{v-structure} in a DAG ${\cal B}$, if directed edges $(V_i,V_k)$ and $(V_j,V_k)$ appear in ${\cal B}$, where $V_i$ and $V_j$ are non-adjacent variables in ${\cal B}$.

The independency information encoded in a DAG can be read graphically by the \emph{d-separation} algorithm in linear time \cite{geigerVermaPearl89}.

\begin{definition}
    \cite{pear88} If ${\bf X}$, ${\bf Y}$, and ${\bf Z}$ are three disjoint subsets of nodes in a DAG ${\cal B}$, then ${\bf Z}$ is said to \emph{d-separate} ${\bf X}$ from ${\bf Y}$, denoted $I({\bf X},{\bf Z},{\bf Y})$, if along every path between a node in ${\bf X}$ and a node in ${\bf Y}$ there is a node $W$ satisfying one of the following two conditions: (i) $W$ is a v-structure and neither $W$ nor any of its descendants are in ${\bf Z}$, or (ii) $W$ is not a v-structure and $W$ is in ${\bf Z}$.
\end{definition}

The next definition formalizes when a DAG is an I-map of a joint probability distribution (JPD).

\begin{definition}
    \cite{darwiche09} Let ${\cal B}$ be a DAG and $P$ be a JPD over the same set of variables.
    ${\cal B}$ is an \emph{I-map} of $P$ if and only if every conditional independence read by d-separation on ${\cal B}$ holds in the distribution $P$.
    An I-map ${\cal B}$ is \emph{minimal}, if ${\cal B}$ ceases to be an I-map when we delete any edge from ${\cal B}$.
\end{definition}

BNs are DAGs with nodes representing variables and edges representing variable dependencies, in which the strength of these relationships are quantified by conditional probability tables (CPTs).
More formally, a BN over variables ${\bf X}$ has its CPTs defined over each variable given its parents, that is, $P(V | Pa(V))$, for every $V \in {\bf X}$.
One salient feature is that the product of the BN CPTs yields a JPD $P$ over ${\bf X}$.

\begin{definition}
    \cite{pear88} Given a JPD $P$ on a set of variables ${\bf X}$, a DAG ${\cal B}$ is called a \emph{Bayesian network} (BN) of $P$ if ${\cal B}$ is a minimal I-map of $P$.
\end{definition}

In a BN, the independencies read by d-separation in the DAG ${\cal B}$ are guaranteed to hold in the JPD $P$.
Unfortunately, while BNs have clear interpretability, exact inference in BNs is NP-hard \cite{coop90}.

BNs can be compiled into \emph{Arithmetic Circuits} (ACs) \cite{darwiche2003differential} by graphically mapping the operations performed when marginalizing all variables from the BN.
\begin{definition}
    \cite{darwiche2003differential} An arithmetic circuit (AC) over variables $\bf U$ is a rooted DAG whose leaf nodes are labeled with numeric constants, called parameters, or $\lambda$ variables, called indicators, and whose other nodes are labeled with multiplication and addition operations.
\end{definition}

Notice that parameter variables are set according to the BN CPTs, while indicator variables are set according to any observed evidence.    

SPNs are a probabilistic graphical model that can be learned from data using, for instance, the LearnSPN algorithm \cite{gens2013learning}.

\begin{definition}
    \cite{poon2011sum} A \emph{sum-product network} (SPN) is a DAG containing three types of nodes: leaf distributions, sums, and products.
    Leaves are tractable distribution functions over ${\bf Y} \subseteq {\bf X}$.
    Sum nodes $S$ compute weighted sums $S = \sum_{N \in Ch(S)}{w_{S,N} N}$, where $Ch(S)$ are the children of $S$ and $w_{S,N}$ are weights that are assumed to be non-negative and normalized \cite{peharz2015theoretical}.
    Product nodes $P$ compute $P = \prod_{N \in Ch(P)}{N}$.
    The value of an SPN, denoted ${\cal S}({\bf x})$, is the value of its root.    
\end{definition}

The \emph{scope} of a sum or product node $N$ is recursively defined as $sc(N) = \bigcup_{C \in Ch(N)}{sc(C)}$, while the scope of a leaf distribution is the set of variables over which the distribution is defined.
A \emph{valid} SPN defines a JPD and allows for efficient inference \cite{poon2011sum}.
The following two structural constraints on the DAG guarantee validity.
An SPN is \emph{complete} if, for every sum node, its children have the same scope.
An SPN is \emph{decomposable} if, for every product node, the scopes of its children are pairwise disjoint.
Valid SPNs are of particular interest because they represent a JPD over the variables in the problem domain.
In addition, like ACs, exact inference is linear in the size of the DAG.
Unlike ACs, however, SPNs allow for a latent variable (LV) interpretation \cite{Peharz:2016wl}.


In \cite{poon2011sum}, it was suggested that SPNs can be interpreted as hierarchical latent variable model, where each sum node corresponds to a latent, marginalized random variable.
This interpretation can be made explicit, by incorporating the latent variables \emph{explicitly} in the SPN, yielding the so-called augmented SPN \cite{Peharz:2016wl}.
We briefly review the construction of the augmented SPN here:
first, for each sum node $S$ we postulate a random variable $Z_S$ with $|Ch({S})|$ states, i.e.~each state of $Z_S$ corresponds to one of $S$'s children.
The states of $Z_S$ are represented via indicators $\lambda_{{Z_S}=k}$, which are explicitly introduced in the SPN.
Furthermore, the sum node $S = \sum_{k} w_{S, k} C_k$ is replaced with $S = \sum_{k} w_{S, k} C_k * \lambda_{Z_S = k}$. 
This construction, originally proposed by \cite{poon2011sum}, allows us to switch $S$'s children ``off and on'' by setting the indicators of $Z_S$, and therefore interpret the children as distributions, conditioned on $Z_S$.

However, as observed in \cite{Peharz:2016wl}, this construction is in conflict with the completeness requirement of SPNs.
Consider a sum node $S^c \in An(S)$ with the property that it has a child $C \in Ch(S^c)$ which does not reach $S$, i.e., $S \notin De(C)$.
\cite{Peharz:2016wl} call such $S^c$ a \emph{contitioning sum} of $S$.
Including the indicators as above renders $S^c$ incomplete, since some but not all children of $S^c$ reach $Z_S$.
This has the severe consequence that the tractable inference mechanism for SPNs is now invalid.
\cite{Peharz:2016wl} propose to fix this problem by introducing a new (dummy) sum node $\bar{S}$, which has only the indicators $\lambda_{Z_S=k}$, $k=1\dots,|Ch(S)|$ as children, the so-called \emph{twin sum} of $S$.
Each child $C \in Ch(S^c)$ which does not reach $S$ is now replaced with $C * \bar{S}$.
As shown in \cite{Peharz:2016wl}, this process leads to a consistent augmentation of the SPN in that sense that it explicitly manifests a random variable for $S$, while i) maintaining completeness and decomposability, and ii) leaving the marginal distribution over observed variables $X$ unchanged.
For further details, see \cite{Peharz:2016wl}, in particular Algorithm AugmentSPN.
\section{SPN Decompilation}
\label{sec:dec}

In this section, we formalize SPN decompilation into a BN.

The interpretation of SPNs as latent variable models is not unique.
For instance, every SPN sum node can be viewed itself as a LV, as done in \cite{zhao2015relationship,peharz2015theoretical}.
In stark contrast, all SPN sum nodes can be interpreted as one single LV \cite{Peharz2015-thesis}.
This provides a wide spectrum of interpretations based only on those sum nodes appearing in an SPN.
In addition, external LVs can be introduced to an SPN, such as the \emph{switching parents} in \cite{peharz2015theoretical}.
Thus, for a given SPN learned from data, there are seemingly countless possible interpretations of its latent space.

The approach taken in this paper is to make a \emph{compilation assumption}, i.e., to assume that there exist some underlying BN which was compiled into the SPN at hand.
While there are many possible ways to compile a BN into an SPN, we use arguably the most prominent compilation method, \emph{Variable Elimination} (VE) \cite{zhan94} following any \emph{reverse topological order} (VErto).
In particular, the recursive marginalization of variables during VE generate hierarchical layers of sum nodes, which typically appear in SPNs.
Besides assuming that the underlying SPN was generated from a BN, we further assume that some of the BN's variables have been removed, that is, marginalized from the model.
The task of decompilation can than be formulated to recover the original BN structure as far as possible.
The main contributions in this paper are to i) provide such an algorithm SPN2BN, and ii) show that it indeed recovers the \emph{morally closed} version of the original BN (see Definition \ref{def:moral} for moral closure).

\begin{algorithm}[tb]
    \caption[]{SPN Compilation Assumption}
    \label{alg:comp_assum}
    \begin{algorithmic}[1]
    \State {\bf BN2SPN}(${\cal B}$)
    \State Let $\sigma$ be a reverse topological ordering of ${\cal B}$ \Comment{VE with reverse topological order (VErto)}
    \State ${\cal C}$ = compile-to-AC-with-VE(${\cal B}$,$\sigma$) \label{ln:comp}
    \State ${\cal S}$ = redistribute-parameters(${\cal C}$) \label{ln:params} \Comment{Convert AC to SPN}
    \State ${\cal S}$ = compile-marginalized-spn(${\cal S}$) \label{ln:marg_spn}
    \State $changed = {\bf true}$
    \While{$changed$}
        \State Let ${\cal S}^\prime$ be a copy of ${\cal S}$
        \State ${\cal S}$ = add-terminal-nodes(${\cal S}$) \label{ln:term_nds}
        \State ${\cal S}$ = remove-products-of-products(${\cal S}$) \label{ln:prod_of_prod}
        \State ${\cal S}$ = lump-products(${\cal S}$) \label{ln:prod_same_ch} \Comment{Lump products over the same children}
        \If{${\cal S} == {\cal S}^\prime$}
            \State $changed = {\bf false}$
        \EndIf
    \EndWhile
    \State Return ${\cal S}$
\end{algorithmic}
\end{algorithm}

Algorithm \ref{alg:comp_assum} describes our compilation assumption, VErto followed by optionally marginalizing some variables.
In line 4, a given BN B is converted into an A ${\cal C}$ using VErto \cite{darwiche2003differential}.
In line 6, the leaf parameters are redistributed as sum-weights \cite{rooshenas2014learning}, yielding an SPN ${\cal S}$.
In line 7, we assume that all internal latent variables in ${\cal S}$ are marginalized and, thus, all of their indicator variables are set to 1.
Here, any arbitrary subset of the internal latent variables can be considered.
Next, we recursively simplify ${\cal S}$ by applying three operations until no further change can be made.
In line 11, a sum node with only indicator nodes as children is converted into a \emph{terminal node} \cite{zhao2015relationship}, which is a univariate distribution over the indicator variable.
Product nodes whose children are exclusively products, i.e., chains of products are simplified into a single product node in line 12.
Finally, in line 14, if two or more product nodes have the same set of children, then they are lumped into a single product node.

On the other hand, by \emph{decompilation}, we mean the procedure of converting an SPN into a BN.
This process involves determining the RVs and DAG for the BN.
We can suggest RVs for the BN by analyzing the compilation assumption.
Similarly, an I-map can be obtained as a DAG using the SPN DAG.
We now formalize these ideas.

\begin{definition}
    Given an SPN over RVs ${\bf X}$ and a compilation assumption, SPN \emph{decompilation} is an algorithm that both: (i) suggests a set of LVs ${\bf Z}$, and; (ii) produces an I-map over ${\bf X}$ and ${\bf Z}$.
\end{definition}

Task (i) of SPN decompilation is more involved than expected.
A naive approach is to disregard the compilation assumption and treat each sum node as one LV.
Negative consequences of this approach will be discussed in the next section.
We suggest a more elegant approach by interpreting the effect of the compilation assumption on the graphical characteristics of the SPN.

Recall that we assume the SPN was compiled using VErto.
During compilation, marginalizing variables creates groups of sum nodes in the same layer (the distance of the longest path from the root).
Hence, identifying these groups is a way of suggesting RVs for the decompiled BN.

More formally, given a sum node $S$, the \emph{sum-depth} of $S$ is the number of sum nodes in the longest directed path from the root to $S$.

\begin{example}
    The sum-depth of sum node $S_3$ in the SPN of Figure \ref{fig:comp_hmm} (\subref{subfig:spn_hmm}) is 2, since there are 2 sum nodes on the longest path from the root to $S_3$.
    Similarly, the sum-depth of $S_2$ is 1 and of $S_1$ is 0.
\end{example}

A \emph{sum-layer} is the set of all sum nodes having the same sum-depth.

\begin{example}
    One sum-layer in the SPN of Figure \ref{fig:comp_hmm} (\subref{subfig:spn_hmm}) consists of $S_3$ and $S_5$, since both $S_3$ and $S_5$ have a sum-depth of 2.
    Furthermore, $S_2$ and $S_4$ form another sum-layer, as does $S_1$ by itself.
\end{example}

A \emph{sum-region} is the set of all sum-nodes within the same sum-layer and having the same scope.

\begin{example}
    Sum-layer $S_3$ and $S_5$ in the SPN of Figure \ref{fig:comp_hmm} (\subref{subfig:spn_hmm}) has only one sum-region, since $S_3$ and $S_5$ have the same scope.
    For the same reason, sum-layer $S_2$ and $S_4$ also has only one sum-region.
\end{example}

A sum-region is created by marginalizing variables during our compilation assumption.
Thus, to answer task (i) of SPN decompilation, we suggest that ${\bf Z}$ consists of one LV per sum-region.

\begin{example}
    In the SPN of Figure \ref{fig:comp_hmm} (\subref{subfig:spn_hmm}), we suggest three LVs in ${\bf Z} = \{Z_{S_1}, Z_{S_2}, Z_{S_3}\}$, namely, one per sum-region.
\end{example}

We now turn our attention to task (ii) of SPN decompilation, that is, constructing an I-map over ${\bf X}$ and ${\bf Z}$.
Augment the SPN as done in \cite{Peharz:2016wl}.
However, before continuing, we need to correct the notion of a conditioning sum node for the following reason.
Consider sum node $S_3$ in the SPN of Figure \ref{fig:comp_hmm} (\subref{subfig:spn_hmm}).
\cite{Peharz:2016wl} would not define sum node $S_1$ as a conditioning sum node for $S_3$, even though $Z_{S_1}$ would appear as a conditioning variable for $Z_{S_3}$ in the CPT $P(Z_{S_3}|Z_{S_1},Z_{S_2},Z_{S_4})$, as depicted in the constructed I-map in Figure \ref{fig:decomp_hmm} (\subref{subfig:peharz_hmm}).

\begin{definition}
    An ancestor sum node $S$ of a node $N$ in an augmented SPN is called \emph{conditioning}, if it is not true that all children of $S$ reach exactly the same subset of $S$ and ${\bar S}$.
    \label{def:cod}
\end{definition}

\begin{example}
    Consider sum node $S_3$ in the SPN of Figure \ref{fig:comp_hmm} (\subref{subfig:spn_hmm}).
    Ancestor sum node $S_2$ is conditioning, since the left-most child of $S_2$ reaches $S_3$, but the right-most child does not.
    Node $S_1$ is not conditioning for $S_3$, since all children of $S_1$ reach the same subset $S_3$ and ${\bar S}_3$ in the augmented SPN.
    \label{ex:conditioning}
\end{example}

In Example \ref{ex:conditioning}, observe that $S_1$ is not a conditioning sum node for $S_3$ and hence $Z_{S_1}$ does not appear as parent of $Z_{S_3}$ in our constructed I-map in Figure \ref{fig:decomp_hmm} (\subref{subfig:us_hmm}).

The SPN decompilation techniques described thus far are formalized as Algorithm \ref{alg:decomp}.

\begin{algorithm}[tb]
    \caption[]{SPN Decompilation}
    \label{alg:decomp}
    \begin{algorithmic}[1]
    \State{\bf SPN2BN}(${\cal S}$)
    \State Let $L$ be the list of sum-layers in ${\cal S}$ 
    \State Let ${\bf Z}$ denote a mapping from nodes to LVs
    \State ${\bf S} = \emptyset$ \Comment{Initialization of Scopes}
    \For{each node $N$ in ${\cal S}$}
        \State ${\bf Z}[N] = \emptyset$
    \EndFor
    \State \Comment{Phase (i) suggests a set ${\bf Z}$ of LVs}
    \For{each $l$ in $L$} \Comment{for each layer}
        \For{each node $N$ in $l$} \Comment{for each node}
            \State Let $X$ denote $scope(N)$
            \If{$X \in {\bf S}$}
                \State Let $Z$ be the existing LV for scope $X$
                \State ${\bf Z}[N] = Z$
            \Else
                \State Let $Z^\prime$ be a new LV
                \State ${\bf Z}[N] = Z^\prime$ \Comment{Update scopes}
                \State ${\bf S} = {\bf S} \cup \{X\}$
            \EndIf
        \EndFor
    \EndFor
    \State \Comment{Phase (ii) produces an I-map over ${\bf X}$ and ${\bf Z}$}
    \For{each node $N$ in ${\cal S}$}
        \If{$N$ is a sum or leaf node}
            \For{each sum node $S^\prime$ in $An(N)$}
                \If{$S^\prime$ is conditioning w.r.t. $N$}
                    \State Add edge $({\bf Z}[S^\prime], N)$ to ${\cal B}$
                \EndIf
            \EndFor
        \EndIf
    \EndFor
    \State Return ${\cal B}$
\end{algorithmic}
\end{algorithm}

\begin{example}
    Given the SPN in Figure \ref{fig:comp_hmm} (\subref{subfig:spn_hmm}) as input to SPN2BN(${\cal S}$), the decompiled BN is given in Figure \ref{fig:decomp_hmm} (\subref{subfig:us_hmm}).
\end{example}

\section{Theoretical Foundation}
\label{sec:th}

In this section, we first establish important properties of SPN decompilation.
Later, we show a favorable characteristic of our compilation assumption and Algorithm \ref{alg:decomp}.

\subsection{On SPN Decompilation}

Our decompilation algorithm is parsimonious with the introduction of LVs.
One LV is assigned per sum-region rather than one per sum node.

Regarding I-map construction, we first show the correctness of the I-map, and then establish that the constructed I-map is minimal.

The I-map correctness follows from the CPT construction suggested in \cite{Peharz:2016wl}.
Theorem 1 in \cite{Peharz:2016wl} shows that certain independencies necessarily hold in an SPN, namely, each LV $Z_S$ for a sum node $S$ is conditionally independent of all non-descendant sum nodes given all ancestor sum nodes of $S$, denoted ${\bf Z}_P$.
A CPT $P(Z_S|{\bf Z}_P)$ is constructed for their I-map.
The I-map built in Algorithm \ref{alg:decomp} uses the same CPT probability values, except building the CPT $P(Z_S|{\bf Z}_C)$, where ${\bf Z}_C$ are those conditioning nodes defined in Definition \ref{def:cod}.
Since ${\bf Z}_C \subseteq {\bf Z}_P$, the independencies encoded in the I-map of \cite{Peharz:2016wl} are a subset of those encoded by the I-map built by Algorithm \ref{alg:decomp}.
Thus, the I-maps proposed in \cite{Peharz:2016wl} are, in general, not minimal.

\begin{example}
    Consider sum node $S_3$ in the SPN of Figure \ref{fig:comp_hmm} (\subref{subfig:spn_hmm}).
    For \cite{Peharz:2016wl}, the CPT for $Z_{S_3}$ is $P(Z_{S_3} | Z_{S_1}, Z_{S_2}, Z_{S_4})$.
    Here, $Z_{S_3}$ is independent of $Z_{S_1}$, given $Z_{S_2}$ and $Z_{S_4}$.
    Thus, Algorithm \ref{alg:decomp} builds the smaller CPT $P(Z_{S_3} | Z_{S_2}, Z_{S_4})$ for $Z_{S_3}$.
\end{example}



The proof for these new conditional independencies and the correctness of our I-map is formalized in Lemma \ref{cor:cond_imap}.

\begin{lemma}
    Consider an augmented SPN ${\cal S}$ with a sum node $S$.
    Let ${\bf A}$ be all of $S$'s ancestors and ${\bf C}$ all of $S$'s conditioning sum nodes.
    Then, $P(Z_S | {\bf Z}_A) = P(Z_S | {\bf Z}_C)$.
\label{cor:cond_imap}
\end{lemma}
\begin{proof}
    Let ${\bf N} = {\bf A} -{\bf C}$ be the non-conditioning ancestor sum nodes of $S$.
    While conditioning on ${\bf Z}_A$ selects a single path in the augmented SPN from the root to $S$, multiple paths may exist from the root to $S$ when conditioning only on ${\bf Z}_C$.
    By Definition \ref{def:cod}, however, all children of a non-conditioning node reach the same subset of $S$ and ${\bar S}$.
    Therefore, the weight $w$ corresponding to the instantiation of ${\bf Z}_S$ necessarily appears in every term of products in the summation.
    By the distributive law, $w$ can be pulled out of this summation of products.
    The resulting summation of products is precisely the SPN computation for the probability of the conditioning event.
    Hence, these two summation of products cancel each other out leaving the conditional probability of $P({\bf Z}_S|{\bf Z}_C)$ to be $w$.
    Thus, $P({\bf Z}_S|{\bf Z}_A) = P({\bf Z}_S|{\bf Z}_C)$.
\end{proof}

We next show that our constructed I-maps are minimal.

\begin{theorem}
    Let $S$ be an augmented SPN over RVs ${\bf X}$ and LVs ${\bf Z}$.
    Given $S$ as input, the algorithm SPN2BN builds a minimal I-map.
\label{the:alg_min_imap}
\end{theorem}
\begin{proof}
    By contradiction, suppose that the constructed I-map ${\cal B}$ is not minimal.
    Then there exists a directed edge $(S_C, S_N)$ from a conditioning node $S_C$ for a sum node $S_N$ that can be deleted without destroying I-mapness.
    In particular, this means that
    \begin{align}
        P(S_N|Pa(S_N)) = P(S_N|Pa(S_N) - S_C). \label{eq:cond}
    \end{align}
    By Definition \ref{def:cod}, as $S_C$ is a conditioning sum for $S_N$, there exist at least two children $S_i$ and $S_j$ of $S_N$ that select different paths to $S_N$ and its twin ${\bar S}_N$.
    Since the respective weights $w_i$ and $w_j$ of $S_i$ and $S_j$ can be different, it immediately follows that Equation (\ref{eq:cond}) is not satisfied by the joint probability distribution $P({\bf X}{\bf Z})$ defined by ${\cal S}$.
    Thus, the I-mapness is violated, if $(S_C,S_N)$ is removed from ${\cal B}$.
    Therefore, by contradiction, ${\cal B}$ is a minimal I-map.
\end{proof}

One seeks minimal I-maps as non-minimal I-maps are not necessarily useful in practice \cite{darwiche09,pear88,koll09}.

\subsection{Compilation and Decompilation}

In this section, we first show that BN2SPN2BN, our compilation-decompilation algorithm, constructs a unique BN for a given set of original BNs.
A consequence of this is that BN2SPN2BN is idempotent.



We next show that the BN output by BN2SPN2BN can be different than the original BN.
In the reminder of this section, we assume $\prec$ is a fixed topological ordering of a given BN ${\cal B}$.

\begin{example}
    Consider the call BN2SPN2BN(${\cal B}$), where BN ${\cal B}$ has directed edges $\{(A,B), (B,E)$, $(C,D), (D,E)\}$.
    Then, the output BN has directed edges $\{(A,B), (B,E), (C,D), (D,E), (B,D)$, $(B,C)\}$.
    \label{ex:recovery}
\end{example}


Notice that the directed edges of the original BN are a subset of those in the output BN.

\begin{definition}
    A \emph{directed moralization edge} is a directed edge $(V_i,V_j)$ added between two non-adjacent vertices $V_i$ and $V_j$ in a given BN ${\cal B}$ whenever there exists a variable $V \in {\cal B}$ such that $V \in Ch(V_i)$ and $V \in Ch(V_j)$, where $V_i \prec V_j$.
\end{definition}




We now introduce the key notion of moral closure.

\begin{definition}
    Given a BN ${\cal B}$ and a fixed topological order $\prec$ of ${\cal B}$, the \emph{moral closure} of ${\cal B}$, denoted ${\cal B}^c$, is the unique BN formed by iteratively augmenting ${\cal B}$ with all directed moralization edges.
    \label{def:moral}
\end{definition}


We are now ready to present the first main result of our compilation-decompilation process.

\begin{theorem}
    Given a BN ${\cal B}$ and a fixed topological order $\prec$ of ${\cal B}$, the output of the compilation-decompilation algorithm BN2SPN2BN is the moral closure ${\cal B}^c$ of ${\cal B}$.
    \label{th:rec_moral}
\end{theorem}
\begin{proof}
    Applying BN2SPN2BN on ${\cal B}$ involves running BN2SPN followed by SPN2BN.
    Consider running BN2SPN on ${\cal B}$.
    This involves eliminating all variables from ${\cal B}$ following an elimination ordering $\sigma$.
    It is well-known that this process builds a triangulated, undirected graph \cite{pear88}.
    Triangulated, undirected graphs admit a perfect numbering.
    This means that no fill-in edges need to be added to the triangulated, undirected graph when eliminating variables following the perfect numbering.

    Now $\sigma$ itself is a perfect numbering for the triangulated graph built by eliminating variable following $\sigma$.
    Let us focus on the undirected edges that were added to ${\cal B}$.
    There are two cases to consider, i.e., moralization edges and triangulation edges.
    A moralization edge corresponds to a directed moralization edge in our case.

    A triangulation edge $(V_i,V_j)$ is added if and only if $V_i$ and $V_j$ are non-adjacent parents of a common child $V_k$ in the directed case.
    However, these edges are precisely the directed moralization edges that are recursively added.
    Therefore, BN2SPN builds an SPN following the hierarchy in the moral closure of ${\cal B}$ and subsequently SPN2BN unwinds the SPN following this same hierarchy.
    Thus, given ${\cal B}$, BN2SPN2BN yields the moral closure ${\cal B}^c$ of ${\cal B}$.
\end{proof}

Theorem \ref{th:rec_moral} has a couple of important consequences.
As Theorem \ref{th:rec_moral} establishes that the output of BN2SPN2BN is the moral closure ${\cal B}^c$ of the input BN ${\cal B}$, it immediately follows that the output BN is exactly the input BN whenever no directed moralization edges are added to ${\cal B}$.
One situation where this occurs is when ${\cal B}$ does not have any v-structures, such as in the case of HMMs.
Here, ${\cal B}^c = {\cal B}$, so the output BN of BN2SPN2BN is the same as the input (up to a relabelling of variables).
For example, recall the HMM in Figure \ref{fig:comp_hmm} (\subref{subfig:bn_hmm}).
BN2SPN2BN yielded back the same BN as illustrated in Figure \ref{fig:decomp_hmm} (\subref{subfig:us_hmm}).
A second important case is when the input is ${\cal B}^c$ itself.
This leads to out next result showing that our compilation-decompilation process is idempotent.

\begin{theorem}
    BN2SPN2BN is idempotent.
    \label{th:idem}
\end{theorem}
\begin{proof}
    Let ${\cal B}$ be a BN and $\prec$ a fixed topological ordering of ${\cal B}$.
    Then ${\cal B}^c = BN2SPN2BN({\cal B})$, by Theorem \ref{th:rec_moral}.
    By definition, the moral closure of ${\cal B}^c$ is ${\cal B}^c$ itself.
    Thus, ${\cal B}^c = BN2SPN2BN({\cal B}^c)$.
    Therefore, BN2SPN2BN is idempotent.
\end{proof}

Theorem \ref{th:idem} has practical significance because it limits the maximum size of the decompiled BN to be the size of the moral closure of the input BN.
In contrast, if we change the decompilation method to \cite{zhao2015relationship} or \cite{Peharz:2016wl}, then applying BN2SPN2BN repeatedly will continually yield a larger BN.


\section{Conclusion}
\label{sec:conclusion}


In this paper, we formalize SPN decompilation by suggesting SPN2BN, an algorithm that converts an SPN into a BN.
SPN2BN is an improvement over \cite{zhao2015relationship} and \cite{Peharz:2016wl}, which are excessive with the number of introduced latent variables.
One key result of our SPN decompilation is that it constructs the moral closure ${\cal B}^c$ of the original BN ${\cal B}$.
This means that in certain cases like for HMMs, where the moral closure of a BN ${\cal B}$ is ${\cal B}$ itself, our SPN decompilation will return the original BN.
Moreover, our compilation-decompilation process is idempotent.
This has practical significance as it limits the maximum size of the decompiled BN to be the size of ${\cal B}^c$.


\begin{thebibliography}{16}
\providecommand{\natexlab}[1]{#1}
\providecommand{\url}[1]{\texttt{#1}}
\expandafter\ifx\csname urlstyle\endcsname\relax
  \providecommand{\doi}[1]{doi: #1}\else
  \providecommand{\doi}{doi: \begingroup \urlstyle{rm}\Url}\fi

\bibitem[Boutilier et~al.(1996)Boutilier, Friedman, Goldszmidt, and
  Koller]{bout96}
C.~Boutilier, N.~Friedman, M.~Goldszmidt, and D.~Koller.
\newblock Context-specific independence in {B}ayesian networks.
\newblock In \emph{Proceedings of the Twelfth Conference on Uncertainty in
  Artificial Intelligence}, pages 115--123, 1996.

\bibitem[Choi and Darwiche(2017)]{choi2017relaxing}
A.~Choi and A.~Darwiche.
\newblock On relaxing determinism in arithmetic circuits.
\newblock In \emph{Proceedings of the 34th International Conference on Machine
  Learning-Volume 70}, pages 825--833, 2017.

\bibitem[Cooper(1990)]{coop90}
G.~Cooper.
\newblock The computational complexity of probabilistic inference using
  {B}ayesian belief networks.
\newblock \emph{Artificial Intelligence}, 42\penalty0 (2-3):\penalty0 393--405,
  1990.

\bibitem[Darwiche(2003)]{darwiche2003differential}
A.~Darwiche.
\newblock A differential approach to inference in {B}ayesian networks.
\newblock \emph{Journal of the ACM}, 50\penalty0 (3):\penalty0 280--305, 2003.

\bibitem[Darwiche(2009)]{darwiche09}
A.~Darwiche.
\newblock \emph{Modeling and Reasoning with {B}ayesian Networks}.
\newblock Cambridge University Press, Los Angeles, CA, 2009.

\bibitem[Geiger et~al.(1989)Geiger, Verma, and Pearl]{geigerVermaPearl89}
D.~Geiger, T.~S. Verma, and J.~Pearl.
\newblock d-separation: From theorems to algorithms.
\newblock In \emph{Proceedings of the Fifth Conference on Uncertainty in
  Artificial Intelligence}, pages 139--148, 1989.

\bibitem[Gens and Domingos(2013)]{gens2013learning}
R.~Gens and P.~Domingos.
\newblock Learning the structure of sum-product networks.
\newblock In \emph{Proceedings of the Thirtieth International Conference on
  Machine Learning}, pages 873--880, 2013.

\bibitem[Koller and Friedman(2009)]{koll09}
D.~Koller and N.~Friedman.
\newblock \emph{Probabilistic Graphical Models: Principles and Techniques}.
\newblock MIT Press, Cambridge, MA, 2009.

\bibitem[Pearl(1988)]{pear88}
J.~Pearl.
\newblock \emph{Probabilistic Reasoning in Intelligent Systems: Networks of
  Plausible Inference}.
\newblock Morgan Kaufmann, San Francisco, CA, 1988.

\bibitem[Peharz(2015)]{Peharz2015-thesis}
R.~Peharz.
\newblock \emph{Foundations of sum-product networks for probabilistic
  modeling}.
\newblock PhD thesis, 2015.

\bibitem[Peharz et~al.(2015)Peharz, Tschiatschek, Pernkopf, and
  Domingos]{peharz2015theoretical}
R.~Peharz, S.~Tschiatschek, F.~Pernkopf, and P.~Domingos.
\newblock On theoretical properties of sum-product networks.
\newblock In \emph{Proceedings of the Eighteenth International Conference on
  Artificial Intelligence and Statistics}, pages 744--752, 2015.

\bibitem[Peharz et~al.(2017)Peharz, Gens, Pernkopf, and
  Domingos]{Peharz:2016wl}
R.~Peharz, R.~Gens, F.~Pernkopf, and P.~Domingos.
\newblock On the latent variable interpretation in sum-product networks.
\newblock \emph{IEEE Transactions on Pattern Analysis and Machine
  Intelligence}, 39\penalty0 (10):\penalty0 2030--2044, 2017.

\bibitem[Poon and Domingos(2011)]{poon2011sum}
H.~Poon and P.~Domingos.
\newblock Sum-product networks: A new deep architecture.
\newblock In \emph{Proceedings of the Twenty-Seventh Conference on Uncertainty
  in Artificial Intelligence}, pages 337--346, 2011.

\bibitem[Rooshenas and Lowd(2014)]{rooshenas2014learning}
A.~Rooshenas and D.~Lowd.
\newblock Learning sum-product networks with direct and indirect variable
  interactions.
\newblock In \emph{Proceedings of the Thirty-First International Conference on
  Machine Learning}, pages 710--718, 2014.

\bibitem[Zhang and Poole(1994)]{zhan94}
N.~L. Zhang and D.~Poole.
\newblock A simple approach to {B}ayesian network computations.
\newblock In \emph{Proceedings of the Tenth Canadian Artificial Intelligence
  Conference}, pages 171--178, 1994.

\bibitem[Zhao et~al.(2015)Zhao, Melibari, and Poupart]{zhao2015relationship}
H.~Zhao, M.~Melibari, and P.~Poupart.
\newblock On the relationship between sum-product networks and {B}ayesian
  networks.
\newblock In \emph{Proceedings of Thirty-Second International Conference on
  Machine Learning}, pages 116--124, 2015.

\end{thebibliography}


\clearpage
\end{document}